\documentclass{ifacconf}

\usepackage{graphicx}
\usepackage{natbib}
\usepackage{amsmath,amssymb}
\usepackage[normalem]{ulem}
\usepackage{float}
\usepackage{tikz}
\usepackage{dblfloatfix}
\usetikzlibrary{decorations.pathreplacing}
\usetikzlibrary{calc}

\usepackage{natbib}        
\newcommand{\R}{\mathbb{R}}
\newcommand{\E}{\mathbb{E}}

\newtheorem{remark}{Remark}

\begin{document}
\begin{frontmatter}

\title{Uniform Error Bounds for Quantized Dynamical Models}


\author[First,Second]{Abdelkader Metakalard} 
\author[Second]{Fabien Lauer} 
\author[First]{Kevin Colin}
\author[First]{Marion Gilson}

\address[First]{Université de Lorraine, CNRS, CRAN, F-54000 Nancy, France  
(e-mail: firstname.lastname@univ-lorraine.fr)}
\address[Second]{Universit\'e de Lorraine, CNRS, LORIA, F-54000 Nancy, France  (e-mail: firstname.lastname@loria.fr)
}

\begin{abstract}                
This paper provides statistical guarantees on the accuracy of dynamical models learned from dependent data sequences. Specifically, we develop uniform error bounds that apply to quantized models and imperfect optimization algorithms commonly used in practical contexts for system identification, and in particular hybrid system identification. Two families of bounds are obtained: slow-rate bounds via a block decomposition and fast-rate, variance-adaptive, bounds via a novel spaced-point strategy. The bounds scale with the number of bits required to encode the model and thus translate hardware constraints into interpretable statistical complexities.
\end{abstract}

\begin{keyword}
System identification, Statistical learning theory, dependent data\end{keyword}

\end{frontmatter}

\sloppy

\section{Introduction}

This paper lies at the intersection of system identification \citep{ljung1999system}, which aims at learning models of dynamical systems, and learning theory \citep{Vapnik1998}, which provides statistical guarantees on the accuracy of models learned from data. More specifically, we concentrate on the question of obtaining high probability bounds on the accuracy of models learned from sequences of possibly dependent data, while taking practical considerations into account. In particular, the proposed analysis holds for imperfect algorithms (for instance those that cannot guarantee to minimize the empirical risk) and models implemented on computing devices with finite precision.   
Indeed, local optimization or heuristic methods are often used for learning in practice, when training neural networks \citep{Goodfellow2016} or estimating hybrid dynamical systems \citep{lauer2019hybrid} for instance, and models are more and more implemented with low precision due to limited hardware capabilities (as with microcontrollers) or latency and power consumption restrictions \citep{jacob2018quantization}.

\subsection{Related Work}

The main difficulty for providing nonasymptotic risk guarantees for system identification stems from the dependence between data points that are collected at subsequent time steps along a single trajectory of the modeled system.  This issue  has been addressed from two complementary perspectives in the literature.

\begin{itemize}
    \item \textbf{Mixing–based learning theory: }
A first family of approaches extends classical  tools from learning theory to dependent data via mixing arguments \citep{Yu94,Meir00,Weyer00,Vidyasagar04,Mohri09,MassucciLauerGilson2022}. These analyses quantify the temporal dependence through coefficients (e.g., $\beta$– or $\theta$–mixing ones) that capture the decay of correlations across time. Then, a decomposition technique due to the seminal work of \citet{Yu94} yields bounds that apply to a subsample of the data. However, the loss in terms of effective sample size is compensated by the versatility of the approach that can yield widely applicable and uniform error bounds, i.e., results that are algorithm-independent. Yet, these approaches typically rely on rather involved measures of the complexity of the model that must be accurately analyzed before applying the bounds, such as Rademacher complexities for \citet{Mohri09}, growth functions for \citet{McDonald2011}, weak-dependence metrics  for \citet{AlquierWintenberger2012}, or an information-theoretic divergence for \citet{Eringis2023}.

    \item \textbf{Algorithm–specific finite–sample analyses without mixing:}
A second line of work provides sharp, problem tailored, guarantees for specific estimators and model classes, that are typically linear or well-structured, using self-normalized martingale tools and related techniques \citep{Simchowitz18,Faradonbeh18,Jedra23}, as surveyed in \citep{Tsiamis23}. When a closed-form expression of the estimator is available, these results yield precise finite-sample rates. Their specialization to a given algorithm-model pair makes them complementary to the uniform, algorithm–agnostic perspective adopted below.
\end{itemize}

Notably, other works also consider a mid-point between these two types of approaches: \cite{Ziemann2022} derives strong guarantees for the specific case of the least-squares estimator under mixing conditions.

An important area of application for the proposed approach is hybrid system identification, as defined in \cite{lauer2019hybrid}, where the data-generating system switches between different subsystems in an unobserved and unknown manner. Beside the issue of dependence, this raises  additional algorithmic difficulties that prevent the application of the algorithmic-specific approaches mentioned above.  Statistical guarantees for hybrid systems were derived in \cite{Chen2022}, but in a slightly different and simplified setting where the data is collected as multiple short and independent trajectories, each generated by a single subsystem, thus alleviating some algorithmic issues and reducing the dependency issue. 
Other works, like \cite{Sattar2021}, propose error bounds for Markov jump systems, but under the simplifying assumption that the switchings are observed or known, in which case the problem becomes more closely related to the identification of multiple independent linear systems and algorithmic-specific approaches can be more easily developed.

\subsection{Contributions}

This paper focuses on the derivation of widely applicable guarantees that take into account practical limitations often encountered in practice, by following the line of work based on mixing arguments. The proposed results take the form of probabilistic error bounds that enjoy the following properties.  
\begin{itemize}
    \item \textbf{Uniform over the model class.} The bounds hold for any model within the predefined class, and thus remain independent of the identification procedure and insensitive to algorithmic or optimization issues. This is particularly crucial for nonlinearly parametrized models, such as neural networks, or hybrid system identification where complex learning problems are solved using heuristics. 
    \item \textbf{Quantization–aware.} The bounds explicitly take the quantization of models into account via a complexity term based on the number of bits used to encode the model class.
    \item \textbf{Generality and Interpretability.} The results cover a broad range of linear, nonlinear and hybrid dynamical models and are directly applicable to new model classes by merely measuring their bit-size. 
    \item \textbf{Fast rates.} We provide a novel decomposition technique that we leverage to obtain fast-rate bounds that are both more efficient and easier to derive than with standard tools. The resulting bounds are also tighter than most results of the literature in many cases.
    \item \textbf{Simple derivations and explicit constants.} The proposed derivations remain simple enough to  yield small and explicit constants, which are otherwise often too conservative or merely ignored in other works.
\end{itemize}

\subsection{Paper Organization}

We first introduces in Sect.~\ref{sec:framework}  the learning framework before establishing a first bound for independent data in Sect.~\ref{sec:static}. Next, we turn to dependent data, with both slow (Sect.~\ref{sec:slow}) and fast (Sect.~\ref{sec:fast}) rate bounds. Finally, Section~\ref{sec:examples} presents numerical experiments that highlight the benefit of the proposed results, before concluding in Sect.~\ref{sec:cl}.

\section{Learning Framework }
\label{sec:framework}

We consider stationary stochastic processes and establish the general framework before introducing quantization considerations.

Let $(X_t, Y_t)_{t \in \mathbb{Z}}$ be a stationary stochastic process taking values in $\mathcal{X} \times \mathcal{Y}$, with $\mathcal{Y}\subset [-r,r]$. We consider model classes $\mathcal{F}$ consisting of functions $f: \mathcal{X} \to \mathcal{Y}$. For a given loss function $\ell: \mathcal{Y} \times \mathcal{Y} \to [0, M]$, we define the risk (generalization error) of a model $f \in \mathcal{F}$ as:
$$L(f) = \mathbb{E}[\ell(Y_t, f(X_t))],$$
where $\mathbb{E}$ denotes the expectation with respect to $(X_t, Y_t)$, which, by stationarity, does not depend on $t$.

\begin{hypo}\label{ass:bounded}
The outputs $Y_t$ are bounded within $[-r,r]$, and the model $f$  is clipped to ensure $f(X)\in[-r,r]$.  
\end{hypo}

Hypothesis~\ref{ass:bounded} ensures that the loss function remains bounded. For instance, the squared loss $\ell(Y_t,f(X_t)) = (Y_t - f(X_i))^2$ is often considered and bounded by $M=4r^2$ under Hypothesis~\ref{ass:bounded}. 
While clipping is a natural operation when the outputs are known to be bounded,  Hypothesis~\ref{ass:bounded} 
basically requires the system to be stable, which constitutes the main limitation of the proposed approach. However, it could be adapted to a more general setting using concentration results for subgaussian or subexponential distributions \citep{Vershynin25}.

Another basic assumption, often satisfied in practice, will be crucial to our framework throughout the paper: 
\begin{hypo}\label{ass:quantized}
The learning algorithm outputs a function $f$ within a parametric model class $\mathcal{F}$ that is implemented on a computer with a $B$-bits representation of real numbers:
\begin{equation}\label{eq:F}
	\mathcal{F} = \{f(\cdot\, ; \mathbf{w}) :  \mathbf{w}\in\mathcal{W}_B\} ,
\end{equation}
where $f(\cdot\, ; \mathbf{w})$ is parametrized by $\mathbf{w}$, $\mathcal{W}\subset\R^p$ is the admissible set of parameters, $\mathcal{W}_B$ is its quantized version over $B$-bits and $p$ is the number of parameters. 
\end{hypo}

Note that we do not require a precise definition of the encoding mechanism for real numbers: the results derived below hold similarly for all encodings based on the same number of bits $B$ (including both floating-point and fixed-point numbers).
 Given a training sample $(X_1, Y_1), \ldots, (X_n, Y_n)$, the empirical risk is:
$$\hat{L}_n(f) = \frac{1}{n} \sum_{i=1}^n \ell(Y_t, f(X_t)).$$

When the data are dependent, we rely on a mixing coefficient to measure this dependence. A stationary process $(Z_t)$ is $\beta$-mixing if its mixing coefficients $\beta(k)$ converge to zero as $k\to\infty$, where, for any $t$:
$$\beta(k) = \sup_{A \in \sigma(Z_{-\infty}^t), B \in \sigma(Z_{t+k}^{\infty})} |\mathbb{P}(A \cap B) - \mathbb{P}(A)\mathbb{P}(B)|,$$
and $\sigma(Z_s^t)$ denotes the $\sigma$-algebra generated by $(Z_s, \ldots, Z_t)$.

Many dynamical systems generate $\beta$-mixing processes with exponentially decaying coefficients. Classical examples studied in \cite{doukhan1995mixing} include:
\begin{itemize}
\item Linear systems: For $Y_{t+1} = AY_t + \eta_t$ with spectral radius $\rho(A) < 1$ and i.i.d. noise $\eta_t$, the mixing rate is governed by the spectral radius: $\beta(k) \leq C\rho(A)^k$ for some constant $C$. In the univariate case ($Y_t \in \mathbb{R}$), this reduces to the autoregressive model $Y_t = \theta Y_{t-1} + \eta_t$ with $|\theta| < 1$, yielding $\beta(k) \leq C|\theta|^k$.

\item Nonlinear autoregressive models: Systems $Y_t = g(Y_{t-1}, \ldots, Y_{t-p}) + \eta_t$ where $g$ satisfies Lipschitz conditions with sufficiently small constant exhibit $\beta$-mixing with exponentially decaying coefficients.
\end{itemize}

\section{Error Bounds for Quantized Models}
\label{sec:static}

We first introduce error bounds for quantized models in the static case before extending to dynamical systems, in order to exhibit how quantization affects statistical guarantees and demonstrate the benefits of our approach in a simpler setting.

Our first generalization error bound below relies on a quantization of the parameter space as in Hypothesis~\ref{ass:quantized} that limits the cardinality of $\mathcal{F}$ to at most $2^{Bp}$, where $p$ is the number of parameters and $B$ the number of bits for each parameter.

\begin{thm}[Error bound for independent data]
\label{thm:finite-class}
Assume $(X_t,Y_t)$ are independent and identically distributed (i.i.d.), and $\mathcal{F}$ is a model class as in Hypothesis~\ref{ass:quantized}. Then, for any $\delta\in(0,1)$, with probability at least $1-\delta$, 
\[
 \forall f\in\mathcal{F},\quad L(f)\ \le\ \hat L_n(f)\ + M\sqrt{\frac{Bp\,\ln 2+\ln(1/\delta)}{n}}.
\]
\end{thm}

\begin{pf}
Let $N = \text{Card }\mathcal{F} \leq 2^{Bp}$. 
Since $\ell$ is bounded in $[0,M]$, for each $f \in \mathcal{F}$, Hoeffding's inequality gives: 
\[
\mathbb{P}\left(  L(f) - \hat L_n(f) \geq \epsilon \right) \leq \exp(-2 n \epsilon^2/M^2).
\]
 Applying the union bound over all $f \in \mathcal{F}$, we get: 
\[
\mathbb{P}\left( \exists f \in \mathcal{F}, L(f) - \hat L_n(f) \geq \epsilon \right) \leq N \exp(-2 n \epsilon^2/M^2).
\]
Setting the right-hand side equal to $\delta$ then yields

\[
\epsilon^2 = M^2\frac{\log(N) + \log(1/\delta)}{2n},
\]
and, since $N \leq 2^{Bp}$, 
\[
\epsilon \leq M\sqrt{\frac{Bp\log 2 + \log(1/\delta)}{n}}.
\]
Thus, with probability at least $1-\delta$, for all $f \in \mathcal{F}$,
\[
L(f) \leq \hat L_n(f) + M\sqrt{\frac{Bp\log 2 + \log(1/\delta)}{n}}.
\]
\end{pf} 

\begin{remark}
The term $Bp$ in the bound of Theorem~\ref{thm:finite-class} corresponds to the total number of bits necessary to specify a quantized model, offering a practical and interpretable complexity measure. This bound provides for instance a direct guideline for selecting $B$ in relation to $n$ and $p$ to balance estimation and quantization (approximation) error. 
\end{remark}

\section{Error Bounds for Quantized Dynamical Models}
\label{sec:slow}

We now extend our analysis to dynamical systems by leveraging the block decomposition technique of \cite{Yu94}. 
The key idea is to decompose the sequence of length $n$ into blocks of size $a$ to limit the dependence between observations taken from the odd blocks only.

Specifically, given a sequence $\left( (X_t, Y_t)\right)_{1\leq t\leq n}$ and two integers $a > 0$ and $\mu > 0$ such that $2a\mu = n$, define the $2\mu$ blocks of length $a$, for $j=1,\dots,\mu$, as
$$
B_j = \left((X_{a(j-1)+1},Y_{a(j-1)+1}), \dots,  (X_{aj},Y_{aj}) \right).
$$

This yields two intertwined sequences of blocks,
\begin{align}\label{eq:blocks}
\mathbf{S}_1& = \mathbf{B}_{1:2:2\mu-1} = (B_1, B_3, B_5, \dots, B_{2\mu-1})\\
\mathbf{S}_2 &= \mathbf{B}_{2:2:2\mu} = (B_2, B_4, B_6, \dots, B_{2\mu})\nonumber
\end{align}
as illustrated in Fig.~\ref{fig:blocks}, which provide the basis for the proof of the following result, as detailed in Appendix~\ref{sec:proofthm2}.

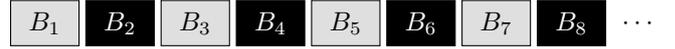
\begin{figure}
\centering
\begin{tikzpicture}[scale=0.9]
  \foreach \i in {1,...,8} {
    \ifodd\i
      \node[draw, minimum width=0.9cm, minimum height=0.6cm, fill=gray!25] at (\i*1.1, 0) {$B_{\i}$};
    \else
      \node[draw, minimum width=0.9cm, minimum height=0.6cm, fill=black, text=white] at (\i*1.1, 0) {$B_{\i}$};
    \fi
  }
  \node at (9.8,0) {$\cdots$};
\end{tikzpicture}
\caption{Two intertwined block sequences: light gray for $\mathbf{S}_1$, black for $\mathbf{S}_2$.\label{fig:blocks}}
\end{figure}

\begin{thm}\label{thm:slow-blocks}
Let $\mathcal{F}$ be a quantized model class as in Hypothesis~\ref{ass:quantized} and, for any $\delta\in(0,1)$, let $\delta' = \delta - 2(\mu-1)\beta(a)> 0$. Then, with probability at least \(1-\delta\),
$$
\forall\, f\in\mathcal{F},\quad L_n(f)  \le \hat L_n(f) + M \sqrt{\frac{2((Bp+1)\ln 2 +\ln\frac{1}{\delta'})}{\mu}}.
$$
\end{thm}

Here again, the complexity of the model is measured in a straightforward manner by the number of bits $B$ and the number of parameters $p$. The dynamical nature of this bound, in comparison with Theorem~\ref{thm:finite-class}, appears in $\delta'$, which includes the mixing coefficient $\beta(a)$, and implies a slightly larger value of the corresponding log term. Another consequence is the fact that the confidence index $\delta$ cannot be set smaller than $2(\mu-1)\beta(a)$. Conversely, the block size $a$ must be properly chosen with the following trade-off in mind. On the one hand, a large value of $a$ reduces the mixing penalty $\beta(a)$, but decreases the effective sample size $\mu = n/(2a)$, which increases the confidence interval $\propto \sqrt{1/\mu}$. On the other hand, a small value of $a$ increases $\mu$ and improves the bound, but also increases $\beta(a)$, which might lead to a violation of $\delta' > 0$. In practice, it is often best to choose the smallest value of $a$ in order to satisfy $\delta'>0$, as will be illustrated by Table~\ref{tab:ar1-ols} in Sect.~\ref{sec:case-ar1}. 

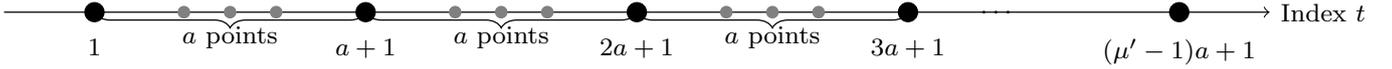
\begin{figure*}[!t] 
  \centering
  \resizebox{\textwidth}{!}{%
    \begin{tikzpicture}[scale=1, every node/.style={font=\small}]
      \draw[->] (0,0) -- (14,0) node[right] {Index $t$};

      \foreach \k/\lab in {0/$1$,1/$a+1$,2/$2a+1$,3/$3a+1$}{
        \coordinate (P\k) at ({1 + 3*\k},0);
        \draw (P\k) ++(0,0.1) -- ++(0,-0.2) node[below,yshift=-2pt] {\lab};
        \filldraw[fill=black, draw=black] (P\k) circle (3pt);
      }

      \foreach \k in {0,1,2}{
        \coordinate (start) at (P\k);
        \coordinate (end)   at (P\the\numexpr\k+1\relax);
        \foreach \t in {0.33,0.5,0.67}{\fill[gray] ($(start)!\t!(end)$) circle (2pt);}
        \draw[decorate, decoration={brace,mirror,amplitude=5pt}]
          (start) -- (end) node[midway, below=1pt] {$a$ points};
      }

      \node at (11,0) {$\cdots$};
      \coordinate (Plast) at (13,0);
      \draw (Plast) ++(0,0.1) -- ++(0,-0.2) node[below,yshift=-2pt] {$(\mu'-1)a+1$};
      \filldraw[fill=black, draw=black] (Plast) circle (3pt);
    \end{tikzpicture}
  }
\caption{Spaced point selection in~\eqref{eq:sp}, showing the first four and the last points, with $a$ points between each pair.}
\label{fig:spaced-points}
\end{figure*}

\section{Fast Rate Bounds}
\label{sec:fast}

While block decomposition is an effective tool for obtaining generalization bounds with dependent ($\beta$-mixing) data, it encounters inherent limitations for achieving fast-rate (variance-dependent) bounds. This is mainly because dependence within each block means that the empirical variance computed on the blocks does not accurately reflect the actual variance of the process, thus preventing meaningful Bernstein-type inequalities. As a result, the obtained rates are pessimistic or the constants overly large, and such bounds rarely improve over the standard “slow rate”. 

To address this issue, we introduce below a new approach based on selecting points that are sufficiently spaced in time, which allows for sharper and more interpretable fast-rate generalization bounds. Indeed, independent copies of the spaced points can be considered, for which the variance entering a Bernstein-type inequality can be easily controlled via the true risk.

Formally, the spaced-points technique, illustrated by Fig.~\ref{fig:spaced-points},  considers a subsample of $\mu'$ spaced points 
\begin{equation}\label{eq:sp}
\mathbf{S}_a = \left((X_{1+(k-1)a}, Y_{1+(k-1)a})\right)_{1\leq k\leq \mu'},
\end{equation}
where $a$ is the spacing parameter and $\mu'= \lfloor n/a\rfloor $ is the effective sample size used to computed the empirical risk
\[
\hat{L}_n^{\text{spaced}}(f) = \frac{1}{\mu'} \sum_{k=1}^{\mu'} \ell(Y_{1+(k-1)a},f(X_{1+(k-1)a})).
\]

\begin{remark}
Notice that since $\mu'$ is equal to $2\mu$, the effective sample size for this approach will be  twice the one obtained by the classical  decomposition into blocks.
\end{remark}

In this setting, the following error bound can be proved, as detailed in Appendix~\ref{app:prooffast}.

\begin{thm}[Fast-rate bound]\label{thm:bernstein-spaced}
Let $\mathcal{F}$ be a quantized model class as in Hypothesis~\ref{ass:quantized}, and, for any $\delta\in(0,1)$, let  $\delta''=\delta-(\mu'-1)\beta(a)>0$. Then, with probability at least $1-\delta$, uniformly over all $f \in \mathcal{F}$, 
\begin{align}
\label{eq:spaced-final}
 \nonumber
L_n(f)
&\le
\hat{L}^{\mathrm{spaced}}_n(f)
+
\sqrt{\frac{2\,M\,Bp\ln(2)+\ln(\frac{1}{\delta''})}{\mu'}\,\hat{L}^{\mathrm{spaced}}_n(f)}\\
& +
\frac{4\,M\,Bp\ln(2)+\ln(\frac{1}{\delta''})}{\mu'}
\end{align}
\end{thm}
The bound of Theorem~\ref{thm:bernstein-spaced} combines two confidence intervals, one in $O(1/\sqrt{\mu'})$ and one in $O(1/\mu')$. However, the first one includes the empirical risk $\hat L_n^{\mathrm{spaced}}(f)$ and vanishes as the model $f$ fits more accurately the data, leading to an effective fast convergence rate of $O(1/\mu')$ in low empirical error cases.  

\begin{remark}[Block decomposition approach and fast rates]
A result in the spirit of Theorem~\ref{thm:bernstein-spaced} could be obtained with the standard block decomposition technique of \citet{Yu94} that we used in Sect.~\ref{sec:slow}. However, the fast rate of  Theorem~\ref{thm:bernstein-spaced} is obtained by taking into account the variance of the process in the derivations. Such an approach based on blocks would have to deal with the variance of blocks, which is itself impacted by the covariance between dependent data points taken from the same block. This would lead to additional terms, an extra level of complexity and an overall less efficient approach than the one we propose above. 
\end{remark}

\section{Examples}\label{sec:examples}

This section presents two example applications of the proposed bounds. 
Section~\ref{sec:case-ar1} considers a simple linear system for which the $\beta$-mixing coefficients can be accurately estimated. This lets us compute the full bound including mixing terms and exhibit the practical gain of the fast-rate bound. Then, Section~\ref{sec:case-hybrid} shows how the proposed approach easily handles hybrid system identification, to which very few others apply beside the one of \cite{MassucciLauerGilson2022}.
\begin{figure}[H]
    \centering
    \includegraphics[width=1\linewidth]{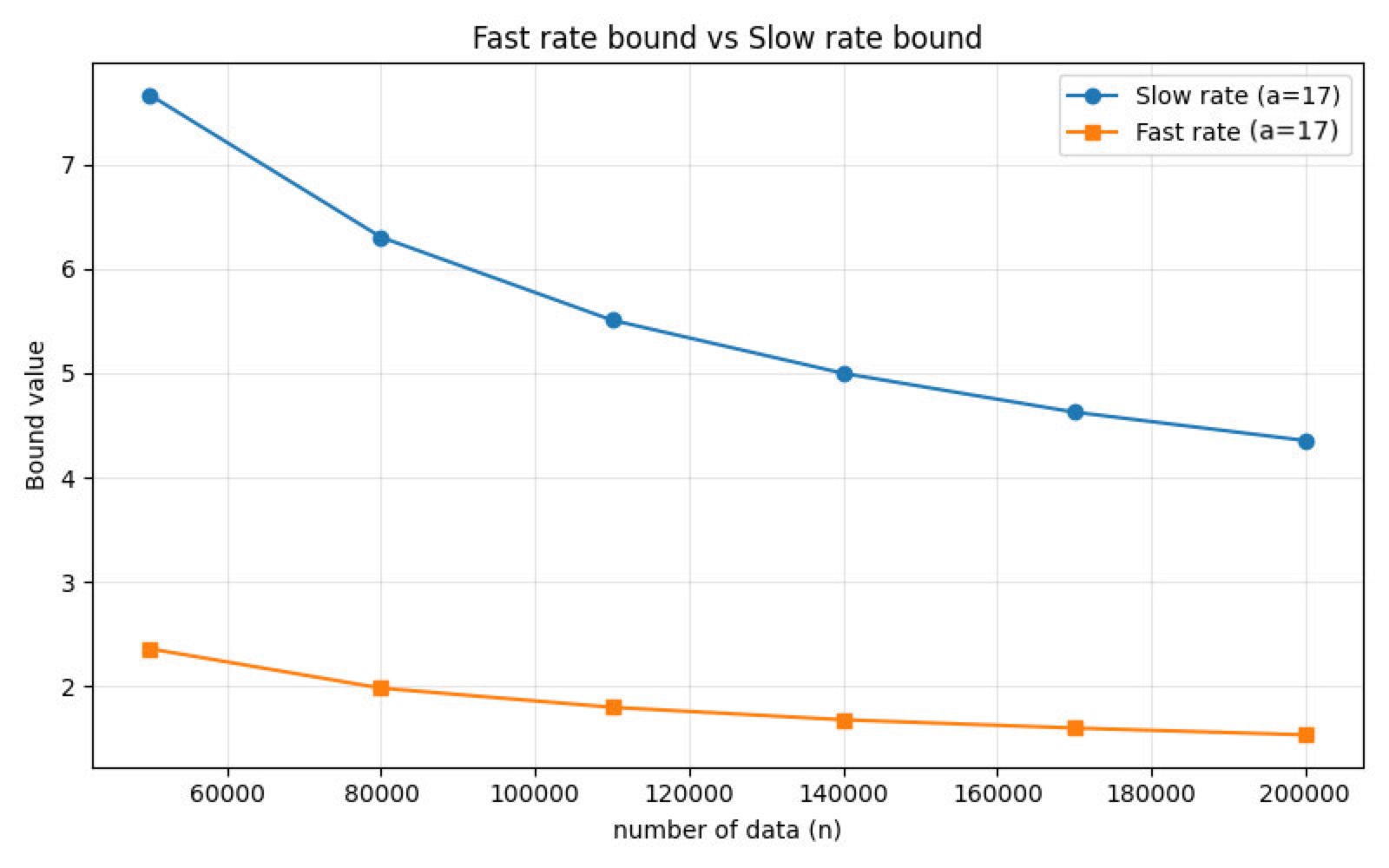}
    \caption{Slow and fast rate error bounds as functions of \(n\) when modeling~\eqref{eq:ex1}. 
    }
    \label{fig:ar1-figure}
\end{figure}
\subsection{Linear system identification}
\label{sec:case-ar1}
\begin{table*}
\centering
\caption{Comparison between the slow and fast rate bounds on the linear system identification example of Sect.~\ref{sec:case-ar1} for $\delta=0.05$ and various block sizes $a$. }
\label{tab:ar1-ols}
\renewcommand{\arraystretch}{1.1}
\begin{tabular}{@{}r|cccc|cccc@{}}
\hline
\multicolumn{1}{c|}{} & \multicolumn{4}{c|}{\textbf{Slow (Theorem~\ref{thm:slow-blocks})}} & \multicolumn{4}{c}{\textbf{Fast (Theorem~\ref{thm:bernstein-spaced})}}
\\
$a$ & $\mu$  & $\hat L_n(f)$ & Confidence interval & Total & $\mu'$ &  $\hat L_n^{\mathrm{spaced}}(f)$ & Confidence interval & Total \\
\hline
 17 & 5882  & 0.9823 & 3.312 & 4.355 & 11764& 0.9869 & 0.630 & 1.691 \\
 21 & 4761  & 0.9823 & 3.701 & 4.683 &  9523& 0.9650 & 0.729 & 1.766 \\
 25 & 3999  & 0.9823 & 4.036 & 5.018 &  7999& 0.9862 & 0.834 & 1.902 \\
 30 & 3333  & 0.9823 & 4.421 & 5.403 &  6666& 0.9820 & 0.954 & 2.031 \\
 40 & 2499  & 0.9823 & 5.105 & 6.088&  4999 & 0.9679 & 1.183 & 2.270 \\
\hline
\end{tabular}
\end{table*}

We first consider a stationary AR(1) time series
\begin{equation}\label{eq:ex1}
  Y_t \;=\; \theta\,Y_{t-1}+\eta_t,
  \qquad \eta_t \stackrel{\text{i.i.d.}}{\sim} \mathcal{N}(0,1),
  \qquad \theta=0.5,
\end{equation}
with \(n=200{,}000\) observations. We learn \(\hat\theta\) by ordinary least squares with $X_t=Y_{t-1}$. 
To keep the squared loss bounded by $M=4r^2$, we clip both data and model outputs at $r=3$. 
We assume a quantized model class with \(p=1\) parameter stored on \(B=32\) bits.
To compute the beta-mixing coeffcient, we use the method of \cite{McDonald2011}.

Table~\ref{tab:ar1-ols} reports the empirical risks (from the same simulation), the confidence intervals, and the overall bound for Theorems~\ref{thm:slow-blocks} and~\ref{thm:bernstein-spaced} for several values of \(a\). Here, the value of $a$ starts at $17$, in order to satisfy the constraints  $\delta'>0$ and $\delta'' > 0$.
These results show that computing the empirical risk on fewer data points, as with $\hat L_n^{\mathrm{spaced}}(f)$, does not significantly impact its value: $\hat L_n^{\mathrm{spaced}}(f)$ is very close to $\hat L_n(f)$ is all tests. Therefore, the fast-rate bound of Theorem~\ref{thm:bernstein-spaced} is always better than the other one. Table~\ref{tab:ar1-ols} also shows that, as $a$ increases, both $\mu$ and $\mu'$ decrease, which results in an increase of the confidence intervals and the overall bounds.
Figure~\ref{fig:ar1-figure} shows how the two bounds decrease monotonically with $n$, as both $\mu$ and $\mu'$ grow linearly with $n$, with rates $O(\frac{1}{\sqrt{\mu}})$ and almost $O(\frac{1}{\mu'})$, respectively.
and
\subsection{Switched system identification}
\label{sec:case-hybrid}

Hybrid systems are systems that switch between different operating modes. Here, we focus on arbitrarily switched linear systems of the form 
\begin{equation}\label{eq:switchedsys}
    y_t = f_{s_t}(x_t) + \eta_t, \qquad 
    f_j(x_t) = w_j^T x_t, \quad j = 1, \ldots, C,
\end{equation}
where $y_t \in \mathbb{R}$ is the output, $x_t \in \mathcal{X} \subset \mathbb{R}^d$ the regression vector, $s_t \in \{1, \ldots, C\}$ the discrete state or mode, $C$ the number of submodels, $f_j$ with $j \in \{1, \ldots, C\}$ the linear submodel of parameters $w_j \in \mathbb{R}^d$ and $\eta_t \in \mathbb{R}$ a noise term. The regressor $x_t \in \mathbb{R}^d$, $d = n_a + n_b$, with the model orders $n_a$ and $n_b$, is given by
\begin{equation}
    x_t = \begin{bmatrix}
    y_{t-1}, \ldots, y_{t-n_a}, u_{t-1}, \ldots, u_{t-n_b}
    \end{bmatrix}^T,
\end{equation}
where the $u_{t-k}$’s denote the delayed inputs. 
In hybrid system identification \citep{lauer2019hybrid}, the switching sequence $(s_t)$ is assumed unknown and the problem is to estimate the submodels $f_j$ from the $(x_t,y_t)$'s only, which is typically done by minimizing the pointwise switching loss  
\begin{equation}
    \ell(f,x,y) = \min_{j\in\{1,\dots,C\}} (y-f_j(x))^2.
    \label{eq:switching-loss}
\end{equation}
Since this loss function is nonconvex (and not differentiable), its minimization is a difficult task often handled by heuristic algorithms for which the estimated model cannot be characterized a priori (except in some specific cases). Thus, statistical guarantees in the flavor of those reviewed by \cite{Tsiamis23} do not apply and {\em uniform} bounds must be considered. The only other approach that provides error bounds in this specific context is the one of \cite{MassucciLauerGilson2022} based on Rademacher complexities. For switched linear systems of the form~\eqref{eq:switchedsys}, it yields 
\begin{align}\label{eq:boundmassucci}
L(f)  \le\ & \hat L_n(f) 
 + \underbrace{  \frac{16 r \Lambda
\sqrt{C \sum_{i=1}^{\mu}\!\|X_{2a(i-1)+1}\|_2^2}}{\mu}}_{\text{Rademacher complexity (mixing–free)}}\\
&+ \underbrace{12r^2\sqrt{\frac{\log(4/\delta''')}{2\mu}}}_{\text{mixing part}},\nonumber
\end{align}
where $\delta''' = \delta-4(\mu-1)\beta(a)$ and $\Lambda$ is an upper bound on the model complexity as measured by $\sqrt{\sum_{j=1}^C \|w_j\|_2^2}$.

Here, we compare bound~\eqref{eq:boundmassucci} with Theorems~\ref{thm:slow-blocks}--\ref{thm:bernstein-spaced}. Since estimating the $\beta$-mixing coefficients remains a complex task out of the scope of this paper, two levels of comparison are considered. One level concentrates on the parts of the bounds that do not depend on the mixing coefficients (as detailed in Appendix~\ref{app:splitbounds}), and another one computes the values of the bounds for a value of the confidence indexes arbitrarily set to $\delta'=\delta''=\delta'''=0.01$. 

The comparison is based on an example switched system taken from \cite{MassucciLauerGilson2022} with $n_a=n_b=2$ and $C=3$ modes of parameters 
\begin{equation}\label{eq:switchedparams}
w_1=\begin{bmatrix}-0.4\\ 0.25\\ -0.15\\ 0.08\end{bmatrix},\quad
w_2=\begin{bmatrix} 1.55\\ -0.58\\ -2.10\\ 0.96\end{bmatrix},\quad
w_3=\begin{bmatrix} 1.00\\ -0.24\\ -0.65\\ 0.30\end{bmatrix},
\end{equation}
input \(u_t\sim\mathcal N(0,1)\), and white output noise $\eta_t$ with a signal-to-noise ratio of \(=30\) dB, over \(n=80\,000\) data points.  The mode $s_t$ is uniformly drawn at random at each time step. Clipping is applied to both data and model outputs with $r=3$. 
Theorems~\ref{thm:slow-blocks}--\ref{thm:bernstein-spaced} are applied with a total number of parameters \(p=Cd=12\), with  \(B=32\) bits per parameter. 
The bound of \cite{MassucciLauerGilson2022} is applied with $\Lambda=\sqrt{\sum_{j=1}^C \|w_j\|_2^2}$ computed with~\eqref{eq:switchedparams} and thus as a tight upper bound on the true model complexity (which is the most favorable case for this bound).
For both approaches, the block length is set to $a=21$, which leads to
$\mu=\lfloor n/(2a)\rfloor=1904$ and $\mu'=\lfloor n/a\rfloor=3809$.
\begin{table}
\centering
\caption{Comparison of error bounds for switched system identification. }
\begin{tabular}{@{}lrr@{}}
\hline
Bound & & Full confidence \\&Complexity term &interval\\
& (mixing-free) & (with mixing)\\
\hline
\cite{MassucciLauerGilson2022},~\eqref{eq:boundmassucci} & 23.48 & 27.77 \\
Theorem~\ref{thm:slow-blocks},~\eqref{eq:slow-bound} & 19.06 & 19.23 \\
Theorem~\ref{thm:bernstein-spaced},~\eqref{eq:fast-bound} & 10.23  & 10.41 \\
\hline
\end{tabular}

\label{tab:bounds-comparison}
\end{table}

\paragraph*{Results.}
Table~\ref{tab:bounds-comparison} reports the mixing-free parts of the confidence interval and the total confidence interval values in the settings discussed above and for the three compared bounds. These results again show the notable advantage of the fast rate of Theorem~\ref{thm:bernstein-spaced}, which also benefits from the twice larger effective sample size $\mu'=2\mu$ that could not be obtained with the standard block decomposition approach. Regarding bound~\eqref{eq:boundmassucci}, the difference between the two reported values reflects the larger constant in front of the mixing part.

\section{Conclusions}\label{sec:cl}

This paper developed uniform error bounds for quantized dynamical models that address key limitations of existing theoretical guarantees. Our results apply to general model classes, account for quantization effects, and provide explicit constants with improved magnitude. The bounds are uniform with respect to the identification algorithm and thus relevant for practical implementations using heuristic optimization methods, as is often the case in e.g. hybrid system identification. In addition, a version with a fast convergence rate was derived with a novel decomposition approach and proved beneficial in numerical experiments.  

Future work may explore extensions to non-stationary processes and adaptive spacing strategies for the novel decomposition technique.

\bibliography{references}

\begin{thebibliography}{23}
\providecommand{\natexlab}[1]{#1}
\providecommand{\url}[1]{\texttt{#1}}
\providecommand{\urlprefix}{URL }
\expandafter\ifx\csname urlstyle\endcsname\relax
  \providecommand{\doi}[1]{doi:\discretionary{}{}{}#1}\else
  \providecommand{\doi}{doi:\discretionary{}{}{}\begingroup
  \urlstyle{rm}\Url}\fi

\bibitem[{Alquier and Wintenberger(2012)}]{AlquierWintenberger2012}
Alquier, P. and Wintenberger, O. (2012).
\newblock Model selection for weakly dependent time series forecasting.
\newblock \emph{Bernoulli}, 18(3), 883--913.

\bibitem[{Chen and Poor(2022)}]{Chen2022}
Chen, Y. and Poor, H.V. (2022).
\newblock Learning mixtures of linear dynamical systems.
\newblock In \emph{International conference on machine learning (ICML)}.
\newblock PMLR.

\bibitem[{Doukhan(1994)}]{doukhan1995mixing}
Doukhan, P. (1994).
\newblock \emph{Mixing: Properties and Examples}.
\newblock Springer-Verlag.

\bibitem[{Eringis et~al.(2024)Eringis, Leth, Tan, Wisniewski, and
  Petreczky}]{Eringis2023}
Eringis, D., Leth, J., Tan, Z.H., Wisniewski, R., and Petreczky, M. (2024).
\newblock {PAC}-bayes generalisation bounds for dynamical systems including
  stable {RNNs}.
\newblock In \emph{Proceedings of the AAAI Conference on Artificial
  Intelligence}, volume~38, 11901--11909.

\bibitem[{Faradonbeh et~al.(2018)Faradonbeh, Tewari, and
  Michailidis}]{Faradonbeh18}
Faradonbeh, M.B., Tewari, A., and Michailidis, G. (2018).
\newblock Finite time identification in unstable linear systems.
\newblock \emph{Automatica}, 96, 342--353.

\bibitem[{Goodfellow et~al.(2016)Goodfellow, Bengio, and
  Courville}]{Goodfellow2016}
Goodfellow, I., Bengio, Y., and Courville, A. (2016).
\newblock \emph{Deep Learning}.
\newblock MIT Press.

\bibitem[{Jacob et~al.(2018)Jacob, Kligys, Chen, Zhu, Tang, Howard, Adam, and
  Kalenichenko}]{jacob2018quantization}
Jacob, B., Kligys, S., Chen, B., Zhu, M., Tang, M., Howard, A., Adam, H., and
  Kalenichenko, D. (2018).
\newblock Quantization and training of neural networks for efficient
  integer-arithmetic-only inference.
\newblock In \emph{Proceedings of the IEEE Conference on Computer Vision and
  Pattern Recognition (CVPR)}, 2704--2713.

\bibitem[{Jedra and Prouti{\`e}re(2023)}]{Jedra23}
Jedra, Y. and Prouti{\`e}re, A. (2023).
\newblock Finite-time identification of linear systems: Fundamental limits and
  optimal algorithms.
\newblock \emph{IEEE Transactions on Automatic Control}, 68, 2805--2820.

\bibitem[{Lauer and Bloch(2019)}]{lauer2019hybrid}
Lauer, F. and Bloch, G. (2019).
\newblock \emph{Hybrid system identification: Theory and Algorithms for
  Learning Switching Models}.
\newblock Springer.

\bibitem[{Ljung(1999)}]{ljung1999system}
Ljung, L. (1999).
\newblock \emph{System Identification: Theory for the User}.
\newblock Prentice Hall PTR.

\bibitem[{Massucci et~al.(2022)Massucci, Lauer, and
  Gilson}]{MassucciLauerGilson2022}
Massucci, L., Lauer, F., and Gilson, M. (2022).
\newblock A statistical learning perspective on switched linear system
  identification.
\newblock \emph{Automatica}, 145, 110532.

\bibitem[{McDonald et~al.(2011)McDonald, Shalizi, and Schervish}]{McDonald2011}
McDonald, D., Shalizi, C., and Schervish, M. (2011).
\newblock Estimating beta-mixing coefficients.
\newblock In \emph{Proceedings of the 14th International Conference on
  Artificial Intelligence and Statistics}, volume~15 of \emph{Proceedings of
  Machine Learning Research}, 516--524.

\bibitem[{Meir(2000)}]{Meir00}
Meir, R. (2000).
\newblock Nonparametric time series prediction through adaptive model
  selection.
\newblock \emph{Machine Learning}, 39, 5--34.

\bibitem[{Mohri and Rostamizadeh(2009)}]{Mohri09}
Mohri, M. and Rostamizadeh, A. (2009).
\newblock Rademacher complexity bounds for non-iid processes.
\newblock In \emph{Advances in Neural Information Processing Systems
  (NeurIPS)}, 1097--1104.

\bibitem[{Sattar et~al.(2021)}]{Sattar2021}
Sattar, Y. et~al. (2021).
\newblock Identification and adaptive control of markov jump systems: Sample
  complexity and regret bounds.
\newblock \emph{arXiv preprint arXiv:2111.07018}.

\bibitem[{Simchowitz et~al.(2018)Simchowitz, Mania, Tu, Jordan, and
  Recht}]{Simchowitz18}
Simchowitz, M., Mania, H., Tu, S., Jordan, M.I., and Recht, B. (2018).
\newblock Learning without mixing: Towards a sharp analysis of linear system
  identification.
\newblock In \emph{Proceedings of the 31st Conference on Learning Theory
  (COLT)}, 439--473.

\bibitem[{Tsiamis et~al.(2023)Tsiamis, Ziemann, Matni, and Pappas}]{Tsiamis23}
Tsiamis, A., Ziemann, I., Matni, N., and Pappas, G.J. (2023).
\newblock Statistical learning theory for control: A finite-sample perspective.
\newblock \emph{IEEE Control Systems Magazine}, 43, 67--97.

\bibitem[{Vapnik(1998)}]{Vapnik1998}
Vapnik, V.N. (1998).
\newblock \emph{Statistical Learning Theory}.
\newblock Wiley.

\bibitem[{Vershynin(2025)}]{Vershynin25}
Vershynin, R. (2025).
\newblock \emph{High-Dimensional Probability}.
\newblock Cambridge University Press, 2nd edition edition.

\bibitem[{Vidyasagar and Karandika(2004)}]{Vidyasagar04}
Vidyasagar, M. and Karandika, R. (2004).
\newblock A learning theory approach to system identification.
\newblock In \emph{Proc. of the 7th IFAC Symposium on Advanced Control of
  Chemical Processes (ADCHEM), Hong Kong, China}, 1--9.

\bibitem[{Weyer(2000)}]{Weyer00}
Weyer, E. (2000).
\newblock Finite sample properties of system identification of arx models under
  mixing conditions.
\newblock \emph{Automatica}, 36(9), 1291--1299.

\bibitem[{Yu(1994)}]{Yu94}
Yu, B. (1994).
\newblock Rates of convergence for empirical processes of stationary mixing
  sequences.
\newblock \emph{The Annals of Probability}, 22(1), 94--116.

\bibitem[{Ziemann and Tu(2022)}]{Ziemann2022}
Ziemann, I. and Tu, S. (2022).
\newblock Learning with little mixing.
\newblock \emph{Advances in Neural Information Processing Systems}, 35,
  4626--4637.

\end{thebibliography}
 
\appendix

\section{Technical Lemmas}

We recall the following seminal result on $\beta$-mixing sequences. 
\begin{lem}[Lemma 4.1 in \cite{Yu94}]
\label{lem:yu} 
Given a sequence of random variables $(Z_i)_{1\leq i\leq n}\in\mathcal{Z}^n$ of mixing coefficients $\beta(k)$, decomposed into blocks as in~\eqref{eq:blocks}, and a bounded function $g : \mathcal{Z}^{a\mu} \to [-\overline{g}, \overline{g}]$,  
\begin{equation}   
\left| \mathbb{E} g(\mathbf{S}_1) - \mathbb{E} g(\mathbf{S}') \right| \leq (\mu - 1) \overline{g} \beta(a),
\end{equation}
where $\mathbf{S}'$ is an independent sequence of blocks with the same marginal distribution for each block as for $\mathbf{S}_1$ but with independent blocks.
\end{lem}

We also provide a slight modification of this result that will be used at the core of the proposed space-points technique, and that holds as a direct consequence of Corollary 2.7 in \cite{Yu94}.
\begin{lem}[Coupling Lemma for Spaced Points]\label{lem:yu-spaced}
Let $(Z_i)_{1\leq i\leq n}$ be a stationary sequence of real‐valued random variables with $\beta$–mixing coefficients $\beta(k)$. Fix an integer spacing $a\ge1$ so that $\mu'=\lfloor n/a\rfloor$ and let $g:\R^{\mu'}\to[-\overline{g},\overline{g}]$ be any bounded function.  Define the \emph{spaced sample} 
\[
\mathbf{S} \;=\;\left(Z_{1},\,Z_{1+a},\,\dots,\,Z_{1+(\mu'-1)a}\right),
\]
and let $\mathbf{S}'=(Z'_1,\dots,Z'_{\mu'})$ be a sequence of independent variables with the same marginal distribution for each $Z_k'$ as $Z_{1+(k-1)a}$.  Then 
\[
\left|\mathbb{E} g(\mathbf{S}) -\mathbb{E} g(\mathbf{S}') \right|
\;\le\;
(\mu'-1)\,\overline{g}\,\beta(a).
\]
\end{lem}

\section{Proof of Theorem~\ref{thm:slow-blocks}}
\label{sec:proofthm2}

Consider the block decomposition of~\eqref{eq:blocks} and, for every $f\in\mathcal{F}$, define the block averages 
\[
h_f(B_j):=\frac{1}{a}\sum_{i=1}^a \ell\big(Y_{a(j-1)+i},f(X_{a(j-1)+i})\big).
\]
and note that each $h_f(B_j)$ is bounded in $[0,M]$ as the loss function. 
For $k\in\{1,2\}$ define the deviations
\begin{align}\label{eq:block-empirical}
\Delta_{\mathbf{S}_k}(f)&:=\E\,h_f(B_1)-\frac{1}{\mu}\sum_{j=0}^{\mu-1}h_f(B_{k+2j}) \nonumber\\
&=L_n(f)-\frac{1}{\mu}\sum_{j=0}^{\mu-1}h_f(B_{k+2j}).
\end{align}
We encode the corresponding large deviation event with the indicator function $\mathbf{1}_{\{\exists f\in\mathcal{F}:\  \Delta_{\mathbf{S}_k}(f)\ge\varepsilon\}}$, which is bounded by $1$. Then, each probability can be computed as
\begin{equation}\label{eq:probaexpect}
\mathbb{P}\{\exists f:\Delta_{\mathbf{S}_k}(f)\ge\varepsilon\} = \mathbb{E} \mathbf{1}_{\{\exists f\in\mathcal{F}:\ \Delta_{\mathbf{S}_k}(f)\ge\varepsilon\}}
\end{equation}
and, by Lemma~\ref{lem:yu}, for each $k\in\{1,2\}$ there exists an i.i.d.\ sequence of blocks $\mathbf{S}_k'$ (composed of i.i.d. blocks with the same marginals as the blocks of $\mathbf{S}_k$) such that
\begin{equation}\label{eq:coupling}
\mathbb{E} \mathbf{1}_{\{\exists f\in\mathcal{F}:\ \Delta_{\mathbf{S}_k}(f)\ge\varepsilon\}}\ \le\ \mathbb{E} \mathbf{1}_{\{\exists f\in\mathcal{F}:\ \Delta_{\mathbf{S}'_k}(f)\ge\varepsilon\}}+(\mu-1)\beta(a).
\end{equation}
On the other hand, since the blocks of $\mathbf{S}'_k$ are independent, Hoeffding’s inequality yields, for any $f$ and any $\varepsilon>0$, 
\[
\mathbb{P}\left\{\Delta_{\mathbf{S}'_k}(f)\ge\varepsilon\right\}
\le \exp\left(-\frac{\mu\,\varepsilon^2}{M^2}\right).
\]
Taking a union bound over $f\in\mathcal{F}$ (with $|\mathcal{F}|=2^{Bp}$) gives
\begin{align*}
\mathbb{P}\left\{\exists f\in\mathcal{F}:\ \Delta_{\mathbf{S}'_k}(f)\ge\varepsilon\,\right\} \le 2^{Bp}\exp\left(-\frac{\mu \varepsilon^2}{M^2}\right).
\end{align*}
Combining with~\eqref{eq:probaexpect} and~\eqref{eq:coupling}, for each $k\in\{1,2\}$,
\begin{equation}\label{eq:odd-or-even}
\mathbb{P}\left\{\exists f\in\mathcal{F}: \Delta_{\mathbf{S}_k}(f)\ge\varepsilon\right\}
\le 2^{Bp}\exp\left(-\frac{\mu\varepsilon^2}{M^2}\right)+(\mu-1)\beta(a).
\end{equation}
By a further union bound applied to \eqref{eq:odd-or-even} for $k=1$ and $k=2$,
\begin{align}\label{eq:max-both}
\mathbb{P}&\left\{\exists f\in\mathcal{F}:\ \max\{\Delta_{\mathbf{S}_1}(f),\Delta_{\mathbf{S}_2}(f)\}\ge\varepsilon\,\right\} \nonumber\\ 
& \le 2^{Bp+1}\exp\left(-\frac{\mu\varepsilon^2}{M^2}\right)+2(\mu-1)\beta(a).
\end{align}
Let $\delta'=\delta-2(\mu-1)\beta(a)>0$ and choose $\varepsilon>0$ so that
$$2^{Bp+1}\exp\!\left(-\frac{\mu\,\varepsilon^2}{M^2}\right)=\delta'
$$ 
This yields
$$\varepsilon=M\sqrt{\frac{2\big((Bp+1)\ln 2+\ln(1/\delta')\big)}{\mu}}.$$
Then,~\eqref{eq:max-both} implies that, with probability at least $1-\delta$, simultaneously for all $f\in\mathcal{F}$,
\[
\Delta_{\mathbf{S}_1}(f)\le\varepsilon\quad\text{and}\quad \Delta_{\mathbf{S}_2}(f)\le\varepsilon.
\]
Using~\eqref{eq:block-empirical}, the two inequalities above are equivalent to
\begin{align*}
&L_n(f)\ \le\ \frac{1}{\mu}\sum_{j=0}^{\mu-1}h_f(B_{1+2j})+\varepsilon \\
\text{and } & L_n(f)\ \le\ \frac{1}{\mu}\sum_{j=0}^{\mu-1}h_f(B_{2+2j})+\varepsilon
\end{align*}
which ensures that
\begin{align*}
L_n(f)\ &\le\ \frac{1}{2}\left(\frac{1}{\mu}\sum_{j=0}^{\mu-1}h_f(B_{1+2j})+\frac{1}{\mu}\sum_{j=0}^{\mu-1}h_f(B_{2+2j})\right)+\varepsilon\\
\ &=\ \hat L_n(f)+\varepsilon.
\end{align*}

\section{Proof of Theorem~\ref{thm:bernstein-spaced}}
\label{app:prooffast}

We start from the set of points introduced in~\eqref{eq:sp} and define the deviation
\begin{align*}
\Delta_{\mathbf{S}_a}(f) &= L_n(f)- \frac{1}{\mu'}\sum_{(X_t,Y_t)\in \mathbf{S}_a} \ell(Y_t,f(X_t))\\
&=L_n(f)-\hat L_n^{\text{spaced}}(f).
\end{align*}
Then, we encode the deviation event by the indicator (bounded by $1$) as
$$
\Psi_{\varepsilon}(\mathbf{S}_a) = \mathbf{1}_{\exists f\in\mathcal{F}:\ \Delta_{\mathbf{S}_a}(f)\ge \varepsilon}
$$
such that 
\begin{equation}\label{eq:c1}
    \E \Psi_{\varepsilon}(\mathbf{S}_a) = \mathbb{P}\left\{\,\exists f\in\mathcal{F}:\ \Delta_{\mathbf{S}_a}(f)
 \ge \varepsilon\,\right\}.
\end{equation}
Successive elements in $\mathbf{S}_{a}$ are separated by exactly $a$ indices of the original
process. By Lemma~\ref{lem:yu-spaced}, there exists an i.i.d.\ sequence
\[
\mathbf{S}_{a}'\ :=\ \big((X_1',Y_1'),\dots,(X_{\mu'}',Y_{\mu'}')\big),
\]
with $(X_j',Y_j')$ distributed as $(X_{1+(j-1)a},Y_{1+(j-1)a})$ such that
\begin{equation}\label{eq:coupling-single-offset}
\E\Psi_{\varepsilon}(\mathbf{S}_a) 
\ \le\ \E \Psi_{\varepsilon}(\mathbf{S}_a') 
\ +\ (\mu'-1)\,\beta(a).
\end{equation}
Fix $f\in\mathcal{F}$ and recall that $\ell(Y_t,f(X_t))$ is bounded in $[0,M]$ with mean $L_n(f)$. Then, 
Bernstein’s inequality implies, for any $\varepsilon>0$,
\begin{align*}
&\mathbb{P}\!\left\{\Delta_{\mathbf{S}_a'}(f)  \ge \varepsilon\right\}
\le \exp\!\left(-\frac{\mu'\varepsilon^2}{2\sigma_f^2+\tfrac23 M\varepsilon}\right),
\end{align*}
where $\sigma_f^2$ is the variance of $\ell(Y_t,f(X_t))$. 

Taking a union bound over $f\in\mathcal{F}$ (with $\text{Card } \mathcal{F}=2^{Bp}$) gives
\begin{equation}\label{eq:iid-single-offset}
\begin{aligned}
\mathbb{P}  \left\{\,\exists f\in\mathcal{F}:\ \Delta_{\mathbf{S}_a'}(f)  \ge \varepsilon\right\}
\le\ 2^{Bp}\exp\!\left(-\frac{\mu'\,\varepsilon^2}{2\,\sigma_f^2+\tfrac23\,M\,\varepsilon}\right).
\end{aligned}
\end{equation}
Combining~\eqref{eq:c1},~\eqref{eq:coupling-single-offset} and~\eqref{eq:iid-single-offset}
yields
\begin{align}\label{eq:main-single-offset}
&\mathbb{P}\left\{\exists f\in\mathcal{F}:\;L_n(f)-\hat L_n^{\text{spaced}}(f)\ge\varepsilon\right\}\\
&\;\le\;
2^{Bp}\,\exp\!\left(-\tfrac{\mu'\,\varepsilon^2}
{2\,\sigma_f^2+\tfrac23\,M\,\varepsilon}\right)
\;+\;
(\mu'-1)\,\beta(a).\nonumber
\end{align}

Set $\delta''=\delta-(\mu'-1)\beta(a)$ and 
$A=\ln(2^{Bp}/\delta'')$.  Enforce
\begin{align*}
&2^{Bp}\exp\left(-\frac{\mu'\,\varepsilon^2}{2\,\sigma_f^2+\tfrac23\,M\,\varepsilon}\right)
=\delta' \ \Leftrightarrow\ 
\frac{\mu'\,\varepsilon^2}{2\,\sigma_f^2+\tfrac23\,M\,\varepsilon}=A.
\end{align*}
Multiplying both sides by \(2\sigma_f^2 + \frac{2}{3}M\,\varepsilon\) leads to
a quadratic equation in \(\varepsilon\):
\[
\mu'\,\varepsilon^2 - \frac{2A}{3}M\,\varepsilon - 2A\sigma_f^2 = 0.
\]
Solving for \(\varepsilon\) using the quadratic formula yields
\[
\varepsilon = \frac{\frac{2A}{3}M + \sqrt{\Bigl(\frac{2A}{3}M\Bigr)^2 + 4\mu'\cdot 2A\sigma_f^2}}{2\mu'}.
\]
Using $\sqrt{S^2+T}\le S+\sqrt T$ with
$S=\tfrac{2A}{3}M$, $T=4\,\mu'\,A\,\sigma_f^2$, 
gives 
\[
\varepsilon
\;\le\;
\frac{2\,M\,A\, }{3\,\mu'}
\;+\;
\sqrt{\frac{2\,\sigma_f^2\,A}{\mu'}}.
\]
Since $\ell(Y_t, f(X_t)) \in[0,M]$, we have the standard envelope–variance bound
\begin{align*}
\sigma_f^2 = Var(\ell(Y_t, f(X_t))) & \le \E[\ell(Y_t, f(X_t))^2] \\
& \le M\E[\ell(Y_t, f(X_t))] = M L_n(f),  
\end{align*}
which yields, together with~\eqref{eq:main-single-offset}, that the following holds with probability at least $1-\delta$: 
\begin{equation}\label{eq:Ineq}
L_n(f) \le \hat{L}^{\mathrm{spaced}}_n(f)
+ \sqrt{\frac{2\,M\,A}{\mu'}\,L_n(f)}
+ \frac{2\,M\,A}{3\,\mu'}.
\end{equation}
Set
\[
b = \sqrt{\frac{2\,M\,A}{\mu'}},
\quad
c = \hat{L}_n(f) + \frac{2\,M\,A}{3\,\mu'}.
\]
Then the inequality~\eqref{eq:Ineq} becomes $L_n(f) \le b\sqrt{L_n(f)} + c$ and by the fact that\footnote{This can be proved by studying the sign of the quadratic polynomial $q(\alpha)=\alpha^2 - b \alpha - c$ for $\alpha=\sqrt{L_n(f)}$, the argument relies on computing the discriminant of $q$ identifying the interval where 
$q(\alpha) \leq 0$ and then squaring the resulting bound on $\alpha$.}
\[
L_n(f) \le b\sqrt{L_n(f)} + c
\quad\Longrightarrow\quad
L_n(f) \le b^2 + b\sqrt{c} + c,
\]
we obtain
\begin{align*}
L_n(f)
\le &
\underbrace{\frac{2\,M\,A}{\mu'}}_{b^2}
\;+\;
\underbrace{\sqrt{\frac{2\,M\,A}{\mu'}}\,
            \sqrt{\hat{L}_n(f) + \frac{2\,M\,A}{3\,\mu'}}}
           _{b\sqrt{c}}\\
&\;+\;
\underbrace{\hat{L}_n(f) + \frac{2\,M\,A}{3\,\mu'}}_{c}.
\end{align*}
Using \(\sqrt{x+y}\le\sqrt{x}+\sqrt{y}\), we write
\[
\sqrt{\hat{L}^{\mathrm{spaced}}_n(f) + \frac{2\,M\,A}{3\,\mu'}}
\;\le\;
\sqrt{\hat{L}^{\mathrm{spaced}}_n(f)}
\;+\;
\sqrt{\frac{2\,M\,A}{3\,\mu'}}.
\]
Hence 
\begin{align*}
b\sqrt{c}
&\leq 
\sqrt{\frac{2\,M\,A}{\mu'}}\;\sqrt{\hat{L}^{\mathrm{spaced}}_n(f)}
\;+\;
\sqrt{\frac{2\,M\,A}{\mu'}}\;\sqrt{\frac{2\,M\,A}{3\,\mu'}}\\
&\leq \sqrt{\frac{2\,M\,A}{\mu'}\,\hat{L}^{\mathrm{spaced}}_n(f)}+\frac{2\,M\,A}{\sqrt{3}\,\mu'}.
\end{align*}
Thus, the terms in $O(1/\mu')$ sum to
\[
2\,M\,A\left(\frac{1}{\mu'} + \frac{1}{3\mu'}+ \frac{1}{\sqrt{3\mu'}}\right)
\approx \frac{3.82\,M\,A}{\mu'} < \frac{4\,M\,A}{\mu}.
\]
Putting everything together, we see that~\eqref{eq:Ineq} implies
\[
L_n(f)
\;\le\;
\hat{L}^{\mathrm{spaced}}_n(f)
\;+\;
\sqrt{\frac{2\,M\,A}{\mu'}\,\hat{L}^{\mathrm{spaced}}_n(f)}
\;+\;
\frac{4\,M\,A}{\mu'},
\]
in which replacing $A$ by its value completes the proof.

\section{Extracting mixing-free parts from the bounds}
\label{app:splitbounds}

The confidence interval in the bound of Theorem~\ref{thm:slow-blocks} can be decomposed into a sum of two terms using the ubadditivity of the square root with:
\begin{align}
&\beta\textbf{ \text{–free part }}=\ M\sqrt{\frac{2\,(Bp+1)\ln 2}{\mu}}
\qquad \label{eq:slow-bound} \\
&\textbf{mixing add-on }=\ M\sqrt{\frac{2\,\ln(1/\delta')}{\mu}}.
\nonumber 
\end{align}

For the bound of Theorem~\ref{thm:bernstein-spaced},
let $A_0= Bp\ln 2$.
Again by $\sqrt{x+y}\le\sqrt{x}+\sqrt{y}$,
\begin{align}
 \beta\textbf{–free part }&=\
\underbrace{\sqrt{\dfrac{2M\hat L^{\mathrm{spaced}}_n A_0}{\mu'}}}_{\text{variance core}}
+\underbrace{\dfrac{4MA_0}{\mu'}}_{\text{linear core}}
\label{eq:fast-bound}
\\[0.35ex]
\textbf{mixing add-ons }&=\ 
\underbrace{\sqrt{\dfrac{2M\hat L^{\mathrm{spaced}}_n\ln\frac{1}{\delta'}}{\mu'}}}_{\text{variance mixing}}
+
\underbrace{\dfrac{4M\ln\frac{1}{\delta'}}{\mu'}}_{\text{linear mixing}}.
\nonumber
\end{align}

\end{document}